\pdfoutput=1
\documentclass[11pt]{article}

\usepackage[utf8]{inputenc}
\usepackage[T1]{fontenc}
\usepackage[margin=1in]{geometry}
\usepackage{amsmath,amssymb,amsthm}
\usepackage{booktabs}
\usepackage{graphicx}
\usepackage{enumitem}
\usepackage[numbers,sort&compress]{natbib}
\usepackage[hidelinks]{hyperref}
\hypersetup{pdftitle={Narrow Decision Models: A System-1 Architecture for Calibrated, Typed Decisions},
  pdfauthor={Carlos Andr\'e Antunes, Aldry Albuquerque, Gabriel de Melo Schmith, Nat\~a Batista, Pedro Oziel Escobar da Silva,
    Stefano Benko Rossi, Lingrid Dietrich Lorenz da Silva, Fl\'avia Cristina da Costa, Thiago Bertotto, Luis Fernando Fausto}}

\newtheorem{proposition}{Proposition}
\newtheorem{definition}{Definition}
\theoremstyle{remark}
\newtheorem*{remark}{Remark}

\newcommand{\anc}{\operatorname{anc}}
\newcommand{\NDM}{\textsc{NDM}}

\title{Narrow Decision Models:\\ A System-1 Architecture for Calibrated, Typed Decisions}

\makeatletter
\renewcommand\@maketitle{%
  \newpage\null\vskip -1.5em
  \begin{center}%
    {\LARGE \@title \par}%
    \vskip 1.1em%
    {\lineskip .35em \begin{tabular}[t]{c}\@author\end{tabular}\par}%
    \vskip .7em%
    {\small \@date}%
  \end{center}%
  \par\vskip .6em}
\makeatother
\renewenvironment{abstract}{\small\begin{center}\bfseries\abstractname\end{center}\vspace{-.6em}\noindent\ignorespaces}{\par}
\newcommand{\autor}[2]{{\small #1}\\[-2pt] {\footnotesize Softplan}\\[-2pt] {\footnotesize\texttt{#2}}}
\author{\autor{Carlos Andr\'e Antunes}{carlos.antunes@softplan.com.br}
  \and \autor{Aldry Albuquerque}{aldry.albuquerque@softplan.com.br}
  \and \autor{Gabriel de Melo Schmith}{gabriel.schmith@softplan.com.br}
  \and \autor{Nat\~a Batista}{nata.batista@softplan.com.br}
  \and \autor{Pedro Oziel Escobar da Silva}{pedro.escobar@softplan.com.br}
  \and \autor{Stefano Benko Rossi}{stefano.rossi@softplan.com.br}
  \and \autor{Lingrid Dietrich Lorenz da Silva}{lingrid.silva@softplan.com.br}
  \and \autor{Fl\'avia Cristina da Costa}{flavia.cristina@softplan.com.br}
  \and \autor{Thiago Bertotto}{thiago.bertotto@softplan.com.br}
  \and \autor{Luis Fernando Fausto}{fausto@softplan.com.br}}

\date{October 2026}

\begin{document}

\maketitle

\begin{abstract}
System One models, introduced with TypeSafe's Jev, answer typed questions
about a state with calibrated probabilities instead of generated text: choosing among options described in natural
language (\emph{Choice}), placing the state on an ordered scale (\emph{Score}), or judging a yes/no statement
(\emph{Noul}). Jev is general-purpose: one set of weights, served as an API to every customer and adapted to a domain
only through the request. We study the other end of the design space. A \emph{Narrow Decision Model} (\NDM) offers the
same interface but is trained for the decisions of one organization, is small enough to run on that organization's own
infrastructure, and has an open architecture whose properties can be stated and tested. The state can be any text or
structured record about which decisions are made; our running case is the routing of court communications. We specify
such an architecture,
its training design and an evaluation protocol. The architecture serializes a request as a prefix tree (the state, then each question, then each option) and runs a
bidirectional encoder under a block attention mask with logical positions, so that the answer to a question is, by
construction, invariant to the order of its options and independent of which other questions are asked with it.
Decisions over hierarchies, such as routing a communication through an organization chart, are compositions of Choice
decisions with an explicit ``none of these'' option. We describe how such a model can be trained with proper scoring rules
through a curriculum that teaches the decision format, then the skill of choosing by reading option descriptions on
fictitious organizations, and only then the target domain; how it is calibrated and allowed to abstain; when a decision
should be composed from other decisions instead of learned; and how its mechanism can be accepted or rejected before any
held-out data is opened. A minimal validation on public and fictitious synthetic data only shows that the invariances
hold to rounding in a trained model, that a small instance decides on organization charts it has never seen well above
lexical retrieval, and that, at equal compute, the block architecture reaches a lower loss than a pairwise
cross-encoder. A first instance trained on real legal documents agrees with the LLM classifier in production on 91\% of
held-out documents; the evaluation on real data is left to a companion paper.
\end{abstract}
\clearpage  %

\section{Introduction}

Organizations make a large number of narrow decisions about the items they handle: documents, support tickets,
transactions, records, the steps of automated workflows. A public legal office that receives a court communication, the
case we build on, must decide which unit handles it, how urgent it is, whether an injunction was granted against the
institution, whether there is a fine or a freeze of public funds, and whether the communication concerns the institution
at all. Each of these is a decision among a small, known set of answers, made over a state (here, a long document) under
a taxonomy that is written down (an organization chart with unit descriptions, a rubric for urgency) and that changes over
time.

General-purpose models are increasingly used for this kind of work, and they come in two kinds. Large language models
can be prompted with the taxonomy and the state and asked for an answer. This works often enough to be tempting, but it has properties that are wrong for a
decision component. The answer is text that must be parsed. The preference among options depends on their order and on
their labels \citep{zheng2024mcq}, and on where the relevant evidence sits in a long context \citep{liu2024lost}. The
confidence, when it is reported at all, is not a calibrated probability. Every additional option or question costs
another pass, and the cost grows with the size of a deliberately general model. In Kahneman's terms
\citep{kahneman2011thinking}, a slow, deliberate System~2 is being used for a task that should be done by a fast, automatic
System~1, but one that, unlike human intuition, knows how confident it is and can say ``I do not know''.

System One models were introduced to close this gap \citep{typesafe2026jev}. Jev, the first of them, receives a
\emph{state} and a set of \emph{typed questions} of three kinds (Choice, Score and Noul), evaluates every question in
parallel and in isolation against a state ingested once, and returns typed answers with probabilities trained to be
calibrated, plus a confidence for Choice and Score; it does not generate text. TypeSafe names its post-training objective
reinforcement learning for calibrated decisions (RLCD) \citep{typesafe2026primer}, but has not published the architecture
or the method. Jev is the second kind of general-purpose model: the same weights serve every customer and every domain,
it is adapted to a domain only through the request (the state, and the instructions and criteria of each question), it is
consumed as a hosted service, and its documentation names English as the language in which it is most accurate. Other
decision models with the same interface followed within weeks, including open encoders and decoders
\citep{convai2026laya,fastino2026glinerdecide,liu2026decision}, and the first independent comparison already shows what documentation
cannot promise: attaching the words yes and no to the opposite rubrics changes about half of the hosted model's answers,
and a classifier fine-tuned on each task is the most accurate model on almost every labeled benchmark
\citep{rafe2026benchmarking}. We take the System One interface as given and ask a different question: what should a
\emph{narrow} model with this interface look like?

A \emph{Narrow Decision Model} (\NDM) is a System One model trained for the decisions of one organization. It receives a
state, which can be any text or structured record (a document, a ticket, a transaction, the context of an agent; here, a
court communication and its metadata), and a set of typed questions, and returns, for each question, a probability
distribution over its admissible answers, in one forward pass, without generating text. Narrowness buys what a
general-purpose service does not offer: the data never leave the organization's infrastructure; the model is trained
on the organization's own taxonomy, rubrics and language (in our case, Portuguese legal text); its versions, costs and
thresholds are under the organization's control; and its architecture is open, so the properties it promises can be
proven and tested instead of taken from documentation. It pays with a training pipeline per domain and with coverage:
outside the decisions it was trained for, a general-purpose model is the better tool. Its options are still described in
natural language and read, not memorized, so that a new organization chart can be served without retraining, and every
choice can include an explicit abstention.

The contributions are architectural and methodological:
\begin{enumerate}[leftmargin=*,itemsep=2pt]
  \item A formal statement of the System One decision contract \citep{typesafe2026jev} for a narrow, trained model,
        with explicit abstention in two regimes, declared and undeclared, and set-valued training targets
        (Section~\ref{sec:contract}).
  \item \NDM-Block, an encoder architecture over a prefix-tree serialization with a block attention mask and logical
        positions. Shared positions and a mask between the options of a single multiple-choice question are known
        \citep{yang2022unimc,egressy2025setllm,convai2026laya}; we extend them to a tree of typed questions over a state
        encoded once, which makes option-permutation invariance, question isolation and exact slicing structural
        properties rather than hopes, and lets the state be cached exactly across sequences (Section~\ref{sec:arch}).
  \item A treatment of \emph{hierarchical decisions} as compositions of Choice decisions in which abstention has a
        meaning (stop at a unit, or triage at the root), extending the greedy and beam searches documented for Jev
        \citep{typesafe2026hier} (Section~\ref{sec:tree}); and of \emph{composed decisions}: where System One practice
        combines atomic answers in code, we push full distributions through a frozen rubric (Section~\ref{sec:composed}).
  \item A \emph{training design}: proper scoring rules applied directly when labels are known, which reach the same
        optimum as the policy gradient of an RLCD-style reward without its variance \citep{maio2026evercld}; a curriculum
        that separates the decision format, the reading skill and the domain; and a synthetic data engine with explicit
        validity guards (Section~\ref{sec:training}).
  \item An \emph{evaluation protocol} that accepts or rejects the mechanism before held-out data is opened
        (Section~\ref{sec:protocol}).
  \item A \emph{minimal validation}, on public and fictitious data only, of the mechanism and of the feasibility of the
        design (Section~\ref{sec:validation}).
\end{enumerate}

Like architecture papers that describe a design before its full evaluation \citep{hinton2021glom,
bengio2017consciousness, schmidhuber2015learning}, we commit here to the design and its reasons. Unlike them, every
property claimed below is implemented and tested in code, and Section~\ref{sec:validation} reports a minimal validation on
public and fictitious data only: mechanism tests on a trained model, the reading of descriptions on public datasets and on
organization charts never seen, calibration, two ablations of the heads and a comparison with the pairwise alternative at
equal compute. Section~\ref{sec:real} summarizes a first training on real legal documents, which was promising; the
evaluation on real data in production is the subject of a companion paper.

\section{Desiderata}
\label{sec:desiderata}

We list the properties a decision component should have. The rest of the paper shows how the architecture provides each
one, and which ones are guaranteed by construction rather than learned.

\begin{description}[style=unboxed,leftmargin=1.2em,itemsep=2pt]
  \item[D1 Typed output.] The output is a distribution over a closed set of answers, never free text.
  \item[D2 Calibration.] The probabilities mean what they say: among answers given probability $0.8$, about $80\%$ are
        right. Calibration is measured, adjusted and reported per question type.
  \item[D3 One pass.] All questions and options of a request are answered in a single forward pass, with the state
        encoded once.
  \item[D4 Option permutation invariance.] Reordering the options of a question permutes its output distribution and
        changes nothing else.
  \item[D5 Question isolation.] The answer to a question does not depend on which other questions are asked in the same
        request, nor on their options.
  \item[D6 Abstention.] Every Choice can include an explicit ``none of these'' option, whose probability is calibrated
        separately from the others.
  \item[D7 Reading, not memorizing.] Options are identified by their descriptions, so that unseen options and unseen
        organization charts can be served without retraining.
  \item[D8 Auditability.] The mechanism can be tested in isolation, with tests that a broken implementation would fail,
        before any held-out data is used.
\end{description}

D1--D3 and D5 are promised at the interface by System One models \citep{typesafe2026jev}. D4 and D5 are the properties
most often violated by prompted language models \citep{zheng2024mcq,liu2024lost}. We want them, and D6--D8, as structural
guarantees of an architecture open to inspection, so that no amount of training data can break them.

\section{The decision contract}
\label{sec:contract}

The contract below is the System One interface \citep{typesafe2026jev}, names of the three question types included,
restated formally, with two additions that a narrow, trained model needs: an explicit abstention option and set-valued
training targets.

\begin{definition}[Request]
A request is a pair $(s, \mathcal{Q})$ where $s$ is the \emph{state}, any text or structured record about which the decisions are made (in our case, the text of a
communication and its structured metadata), and
$\mathcal{Q} = \{q_1, \dots, q_n\}$ is a set of typed questions. Each question has an instruction $\iota_i$ in natural
language and a type:
\begin{itemize}[leftmargin=*,itemsep=1pt]
  \item \textbf{Choice}: a set of options $\{(o_{ij}, d_{ij})\}_{j=1}^{K_i}$ with names $o_{ij}$ and natural-language
        descriptions $d_{ij}$, optionally including a designated abstention option $o_{i\varnothing}$;
  \item \textbf{Score}: an ordered list of levels $\ell_{i1} \prec \dots \prec \ell_{iL}$, each with a description;
  \item \textbf{Noul}: a yes/no judgment, optionally with descriptions of what \emph{yes} and \emph{no} mean.
\end{itemize}
\end{definition}

\begin{definition}[Answer]
The model returns, for each Choice question, a distribution $p_i \in \Delta^{K_i - 1}$ over its options; for each Score
question, a distribution over its levels; for each Noul question, a probability $p_i(\text{yes}) \in [0,1]$. Training
targets may be a single accepted option, a set $A_i$ of accepted options (any of which is correct), or a target
distribution.
\end{definition}

Abstention appears in two regimes. In the \emph{declared} regime the request marks one option as ``none of these''; in the
\emph{undeclared} regime the same option is present as an ordinary option and the model must recognize it from its
description. Both regimes appear in training, so that the model does not learn to rely on the declaration: in the
synthetic and public data the regime is drawn per question from a hash of the question key, independently of the label,
so that about half of the questions declare it. The abstention option has a fixed description in each data source, so
robustness to paraphrases of ``none of these'' is not trained; Section~\ref{sec:limitations} lists it as a limitation.

A typical request concerning a court communication contains one Choice per node of an organization chart along a path
(``which of these units is competent?''), a Choice for the legal matter, a Choice for the type of act required (contest,
appeal, comply, pay, take notice, \dots), a Score for urgency, and two Noul questions (``was an injunction granted against
the institution?'', ``is there a fine or a freeze of public funds?''). All are answered together.

Nothing in the contract, or in the architecture of Section~\ref{sec:arch}, depends on the state being a document.
Structured records are linearized one field per line (its path, then its value) and read like any other text, so the same
interface serves a support ticket, a transaction, a patient record or the context of an agent's next step. Court
communications are the case we build on and the source of our examples; the validation of
Section~\ref{sec:validation} already reads short public intent sentences as states.

\section{Architecture}
\label{sec:arch}

Figure~\ref{fig:overview} summarizes the architecture; the subsections below take its parts in turn.

\begin{figure}[t]
  \centering
  \includegraphics[width=\linewidth]{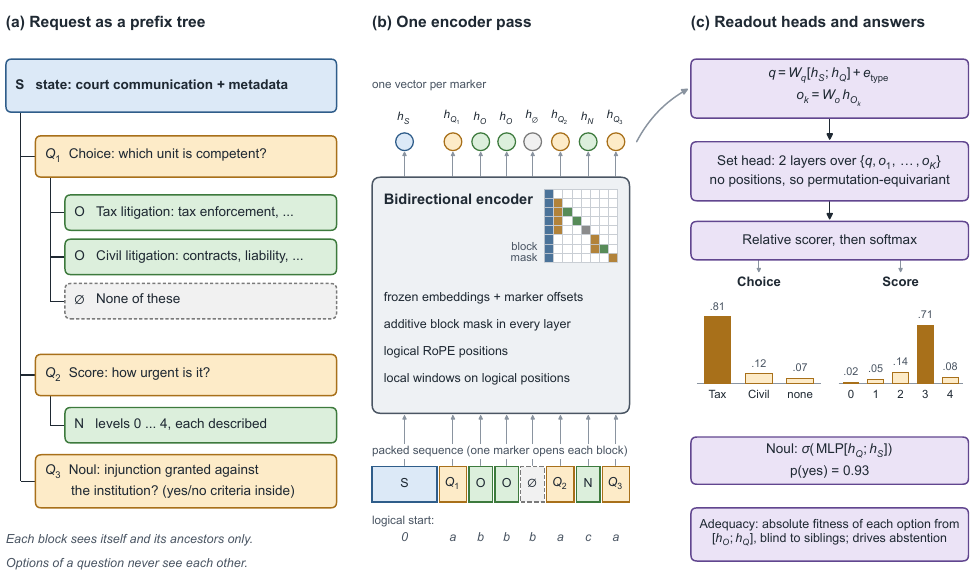}
  \caption{Overview of \NDM-Block. (a)~A request is a prefix tree: the state, the typed questions under it, and the options or levels under each question; the yes/no criteria of a Noul question sit inside its instruction block. (b)~The tree is packed into one sequence and encoded once under a block mask. Positions are logical: every question starts at the same position~$a$, right after the state, and sibling options share position~$b$. (c)~One vector per marker feeds the readout heads, which return one distribution per question. All probabilities in the figures are illustrative.}
  \label{fig:overview}
\end{figure}

\subsection{Backbone}

The backbone is a bidirectional Transformer encoder \citep{vaswani2017attention,devlin2019bert}. We use a modern long-context, multilingual encoder of the
ModernBERT family \citep{warner2024modernbert,marone2025mmbert}, with rotary position embeddings
\citep{su2021roformer} and alternating global and local attention layers. An encoder rather than a decoder is chosen for
three reasons: the contract never generates text; within a block, reading in both directions is what the task requires;
and an encoder of this size costs several times less per token to train than a decoder of comparable quality. The token
embedding table is frozen; the transformer body and a small set of new parameters are trained.

\subsection{Prefix-tree serialization}

A request is serialized as a tree of \emph{blocks}. The root block holds the state $S$. Each question contributes an
instruction block $Q_i$ whose parent is $S$, and each option or level contributes a block $O_{ij}$
whose parent is $Q_i$. Each block starts with a \emph{marker} token that identifies its role (state, Choice instruction,
Score instruction, Noul instruction, option, level, description); a Noul instruction also carries its \emph{yes} and
\emph{no} criteria inside its block, each introduced by a marker. Markers use reserved, unused vocabulary entries. Since
the embedding table is frozen, each marker receives a trainable offset, initialized so that the marker starts as the mean
embedding of a short equivalent phrase (``Option:'', ``Level:'').

The blocks of one request are laid out contiguously in one sequence when they fit; a question with too many options is
\emph{sliced}: its options are distributed over several sequences, each carrying a copy of $S$ and $Q_i$. Several requests
are packed into a batch by length buckets. Slicing keeps each option whole, fills each sequence up to the largest
length bucket and repeats the state and the instruction in every slice; a request is never split across batches.
Appendix~\ref{app:montagem} gives the procedure in pseudo-code, with its cost.

\subsection{Block attention mask with logical positions}

\begin{figure}[t]
  \centering
  \includegraphics[width=\linewidth]{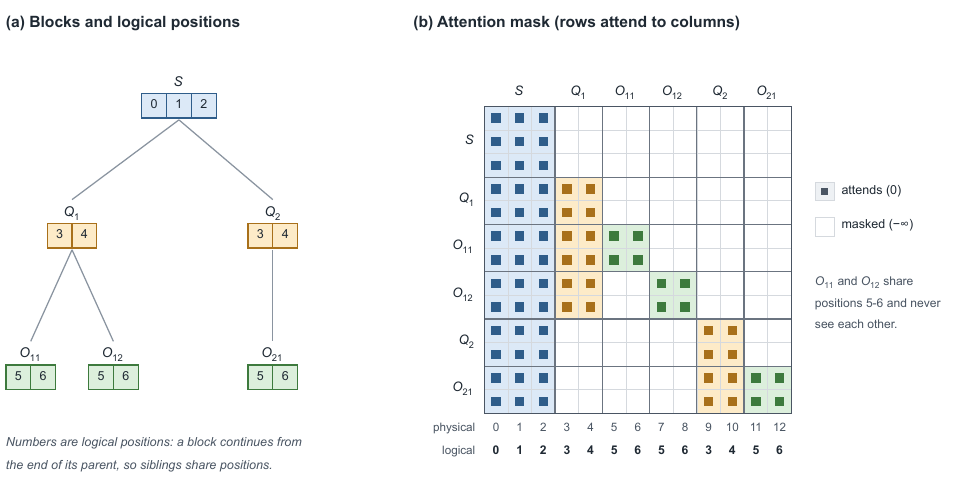}
  \caption{Block attention mask for a request with two questions, two options under $Q_1$ and one under $Q_2$. (a)~The logical positions of a block continue from the end of its parent. (b)~Token-level mask: each row attends to its own block and to its ancestors. The siblings $O_{11}$ and $O_{12}$ occupy the same logical positions and never see each other, which is what makes Propositions~1--3 hold by construction.}
  \label{fig:mask}
\end{figure}

Let $\anc(b)$ be the set of ancestors of block $b$ in the tree. Every token of block $b$ may attend to the tokens of $b$
and of $\anc(b)$, and to nothing else (Figure~\ref{fig:mask}):
\begin{equation}
  M_{tu} =
  \begin{cases}
    0 & \text{if } \mathrm{block}(u) \in \{\mathrm{block}(t)\} \cup \anc(\mathrm{block}(t)),\\
    -\infty & \text{otherwise}.
  \end{cases}
\end{equation}
The mask is additive and applied in every layer. Positions are \emph{logical}: a block's positions continue from the end of
its parent, not from the end of whatever precedes it physically. Two sibling options therefore occupy the same logical
positions, and each sees exactly the context it would see if it were the only option of the only question. Local
attention windows are measured on logical positions, so the same holds for the layers with sliding-window attention.

Under this mask, the state is encoded once and attended by every question; each instruction is encoded once and attended
by its options; and no option, instruction or question sees another one. Tree attention masks of this kind are used to
verify candidate continuations in parallel in speculative decoding \citep{miao2023specinfer,cai2024medusa}, and shared
prefixes are exploited for throughput in decoder serving \citep{juravsky2024hydragen}. Closer to our use, a
multiple-choice encoder has given each option its own position ids and blocked attention between options so that they
do not influence each other \citep{yang2022unimc}; Set-LLM \citep{egressy2025setllm} adapts a decoder with a set mask and
set positions that guarantee invariance to the order of a set; Parallel Context Windows \citep{ratner2023pcw} and PINE
\citep{wang2025pine} reuse positions or restrict attention across segments of a decoder's context; and the open decision
model Laya offers a layout in which the options of a question share positions and cannot see each other
\citep{convai2026laya}. All of these act on one list of candidates or segments, next to a context that may attend to
them. Here the tree is the \emph{semantics} of a whole request, not an inference trick: options within a question and
questions within a request are isolated alike, the state attends to nothing but itself, and the same mask is used in
training as well as in inference.

\subsection{Readout heads}

From the final hidden states we read one vector per marker: $h_S$ for the state, $h_{Q_i}$ for each instruction and
$h_{O_{ij}}$ for each option. A question vector is a projection of $[h_S; h_{Q_i}]$ plus an embedding of the question
type; an option vector is a projection of $h_{O_{ij}}$.

\paragraph{Set head.} The option vectors of one question, together with its question vector, pass through a small
transformer with no positional information, whose attention is restricted to the options of that question. This is the
only place where options of the same question interact: the model can compare candidates, which a pairwise scorer cannot,
while the head remains permutation-equivariant \citep{zaheer2017deepsets}. Set-Encoder \citep{schlatt2025setencoder}
obtains permutation-invariant interactions between candidate passages inside a re-ranking cross-encoder; the set head
obtains them after the encoder, over one vector per option, so that the encoder pass keeps its isolation. Its output
projections are initialized at zero, so that at the start of training it is the identity.

\paragraph{Choice and Score.} A relative scorer maps each option (or level) vector, together with the question vector,
to a logit, through a body shared by both types and one final reading per type; Choice applies a softmax over the options
of the question, Score a softmax over its levels.

\paragraph{Adequacy.} A separate reading gives each option an absolute logit of fitness, computed from the encoder
vectors of the option and the instruction before the set head, and therefore independent of the siblings. It supports
abstention that does not depend on which other options happen to be offered (Section~\ref{sec:calibration}).

\paragraph{Noul.} A Noul question is read from its instruction and state vectors, $[h_{Q_i}; h_S]$, into a single logit
passed through a sigmoid; its yes/no criteria are read through $h_{Q_i}$, since they sit inside the instruction block.

The new parameters (marker offsets, projections, set head and readings) amount to 6.4 million, against a trained body of
110 million and a frozen embedding table of 197 million; the set head has two pre-norm layers of width 384 with six
heads. Appendix~\ref{app:hiper} lists every size and weight.

\subsection{Invariances by construction}

\begin{proposition}[Option permutation]
Let $\pi$ permute the options of question $q_i$. Then the output distribution of $q_i$ is permuted by $\pi$, and the
outputs of all other questions are unchanged.
\end{proposition}
\begin{proof}[Proof sketch]
Sibling options share the same parent and the same logical positions, and the mask forbids attention between them; the
hidden states of each option block are therefore a function of the option's own tokens and of its ancestors only, and do
not depend on its physical index. The set head has no positional information and is equivariant. Other questions cannot
attend to the options of $q_i$.
\end{proof}

\begin{remark}
Proposition~1 needs only the logical positions. Even without the mask, exchanging two sibling blocks that carry the same
logical positions only reorders the keys, which attention sums over without regard to order. Sequential positions break
it. Within one question, this is the invariance that shared option positions already give in the multiple-choice and
set-aware models cited above. The mask is what Propositions~2 and~3 need.
\end{remark}

\paragraph{Relaxing isolation, auditably.} Isolation is a requirement, not an accident: an answer that depends on which
other questions were asked cannot be audited question by question. When a decision should depend on another (a
structural node on the thematic unit chosen above it, a consistency constraint between two answers), the dependency is
made explicit instead of being left to attention: the earlier answer is written into the state of the later request, or
the answers are combined by a rubric (Section~\ref{sec:composed}). Both keep the dependency visible and testable, which
an attention path between questions would not.

\begin{proposition}[Question isolation]
The output for $q_i$ in a request $(s, \mathcal{Q})$ equals its output in the request $(s, \{q_i\})$.
\end{proposition}
\begin{proof}[Proof sketch]
The blocks of $q_i$ attend only to themselves and to $S$; $S$ attends only to itself. No block of another question is an
ancestor of a block of $q_i$, so removing it changes no hidden state that $q_i$ reads.
\end{proof}

\begin{proposition}[Slicing and paths]
Splitting the options of a question over several sequences, or evaluating each root-to-leaf path of the tree as a separate
sequence, yields the same outputs as evaluating the whole tree in one sequence.
\end{proposition}
\begin{proof}[Proof sketch]
By the argument of Proposition~1, the hidden states of a block depend only on its own tokens and on those of its
ancestors, at their logical positions. A slice or a path contains every ancestor of each block it holds, at the same
logical positions; padding tokens belong to no block and are never attended.
\end{proof}

These statements hold exactly in exact arithmetic. In floating point, the order of reductions changes with the shape of
the batch, so they hold up to rounding; Section~\ref{sec:protocol} states how the tolerance is set and tested.

\subsection{Encoding the state once across sequences}

Within a sequence the state is encoded once. Across sequences (slices, successive steps of a tree descent, a beam), it
would be re-encoded. Because the state attends only to itself, its keys and values in every layer (with rotary positions
already applied) can be cached and reused: a question and its options are then encoded as a short query sequence whose
keys and values are the cached state followed by the question's own tokens, under a rectangular version of the same mask.
An additional question then costs only its own tokens in the projections and feed-forward layers; its attention still
reads the cached keys, so that part of the cost grows with the length of the state, but nothing of the state is
recomputed. Reusing the keys and values of separately encoded segments is the idea of Prompt Cache
\citep{gim2024promptcache} and Block-Attention \citep{ma2025blockattention}, where it approximates the full pass or
requires fine-tuning the model to tolerate it; here the reused segment is the state, which attends to nothing else in
training as in inference, so the reuse is exact. The result must reproduce the full forward pass, a property that is
tested, not assumed.

\subsection{Cost compared with the alternatives}

Table~\ref{tab:cost} compares \NDM-Block with two alternatives for the same contract: a pairwise cross-encoder
(\NDM-Pair), which encodes the state, a question and one option at a time, as in passage re-ranking
\citep{nogueira2019passage} or entailment-based zero-shot classification \citep{yin2019benchmarking}; and a decoder that
scores each option by its log-likelihood.

\begin{table}[t]
  \centering
  \small
  \begin{tabular}{@{}lccc@{}}
    \toprule
                                     & \NDM-Block           & \NDM-Pair             & Decoder scoring      \\
    \midrule
    Forward passes per request       & $1$                  & $\sum_i K_i$          & $\sum_i K_i$\,$^{a}$ \\
    Times the state is encoded       & $1$                  & $\sum_i K_i$          & $1$\,$^{a}$          \\
    Options compared with each other & yes (set head)       & no                    & no                   \\
    Option-order invariance          & by construction      & by construction       & no                   \\
    Question isolation               & by construction      & by construction       & prompt-dependent     \\
    Trained for calibration          & yes                  & yes                   & no                   \\
    \bottomrule
  \end{tabular}
  \caption{The same decision contract served by three architectures. $K_i$ is the number of options of question $i$.
  $^{a}$With prefix caching, one pass over the shared prefix and one short pass per option.}
  \label{tab:cost}
\end{table}

\NDM-Pair is the natural baseline: it has the same invariances and no mask, at a cost that grows with the number of
options. It is the reference against which the block mask must be shown not to lose accuracy (Section~\ref{sec:protocol}).

\section{Calibration, abstention and selective decisions}
\label{sec:calibration}

\paragraph{Training with proper scoring rules.} A probability can only be calibrated if the loss rewards honest
probabilities. Every primitive is trained with a sum of two strictly proper scoring rules \citep{gneiting2007strictly}, the log-loss
and a quadratic score: for Choice, the log-loss of the accepted set, $-\log \sum_{k \in A} p_k$, plus the Brier score
\citep{brier1950verification} over the options with the accepted ones merged into one; for Score, the log-loss plus the
ranked probability score \citep{epstein1969scoring}, which penalizes mass placed far from the true level more than mass
placed next to it; for Noul, the binary log-loss plus the Brier score. The log-loss gives strong gradients to confident
mistakes; the quadratic term is bounded and keeps a few mislabeled examples from dominating. The adequacy reading
receives a small auxiliary binary loss.

The log-loss is the primary term: it is strictly proper and local, and it is the more efficient criterion for telling good
forecasters from bad ones \citep{wheatcroft2019rps}. The RPS is added for Score only, as a non-local term that rewards
residual mass near the true level; a cumulative link model \citep{mccullagh1980ordinal} would be the alternative that
constrains the shape of the ordinal distribution instead. Whether the RPS term helps is an ablation of the protocol, not an
assumption of the design.

\paragraph{Post-hoc temperature.} After training, one temperature per primitive (and a bias for Noul) is fitted on a
held-out calibration set by minimizing the negative log-likelihood \citep{platt1999probabilistic,guo2017calibration}. Two
safeguards matter in practice. First, when the calibration set is separable (every prediction correct), the optimal
temperature tends to zero and would make the model arbitrarily overconfident; the temperature is therefore bounded, and
the report states when it was not identified by the data. Second, the temperature fitted on synthetic data is the primary
one; a domain temperature fitted on the small labeled development set is shrunk towards it in log space, with a weight
that falls as the bootstrap uncertainty of the development estimate grows (and, when that uncertainty exceeds the prior
tolerance, it may only make the model less confident). It is reported next to the synthetic one and never replaces it.

\paragraph{Abstention prior.} The rate at which ``none of these'' is correct differs between training data and
deployment. After temperature scaling, the log-odds of the abstention option are shifted from the reference rate
$\rho_{\mathrm{ref}}$ observed in calibration to the production rate $\rho_{\mathrm{prod}}$, when the latter is known.

\paragraph{Selective decisions.} A calibrated model can trade coverage for risk \citep{geifman2017selective}. Two
thresholds are chosen per artifact on a synthetic selection set, each as the largest coverage whose risk stays below a
target: $\tau_p$ on the top calibrated probability at a tree node (below it, the descent stops at the current node), and
$\tau$ on the highest adequacy among the non-abstention options (below it, the model abstains). Decisions that fall below
the thresholds go to a person, as in the confidence-gated routing recommended for System One models
\citep{typesafe2026jev}. A decision component that knows when not to decide is more useful than one that is slightly more
accurate on average.

\paragraph{One decision rule per decision.} The two thresholds read different scores and govern different decisions, so
they are never combined. At a node of the tree, the descent follows the argmax of the calibrated distribution, which
includes the explicit abstention option, and stops if that argmax is the abstention or if its probability is below
$\tau_p$. For a question outside the tree, the model abstains if the highest adequacy among the non-abstention options is
below $\tau$, and otherwise answers with the argmax of the calibrated distribution restricted to those options. Which of
the three readings of abstention (the explicit option, the adequacy, the fallback to the parent) is the most reliable is
measured, not assumed.

\paragraph{Guarantees.} The thresholds above are empirical cuts on a selection set. Methods with finite-sample
guarantees fit in the same place: selective classification with guaranteed risk \citep{geifman2017selective}, and
learn-then-test and conformal risk control \citep{angelopoulos2021ltt,angelopoulos2022crc,angelopoulos2021gentle},
which replace the empirical cut by one with a high-probability bound on the risk. Their guarantees assume that
calibration and deployment data are exchangeable, which synthetic calibration data cannot provide; they belong where real
calibration data exist, in deployment.

\section{Hierarchical decisions over organization trees}
\label{sec:tree}

\begin{figure}[t]
  \centering
  \includegraphics[width=\linewidth]{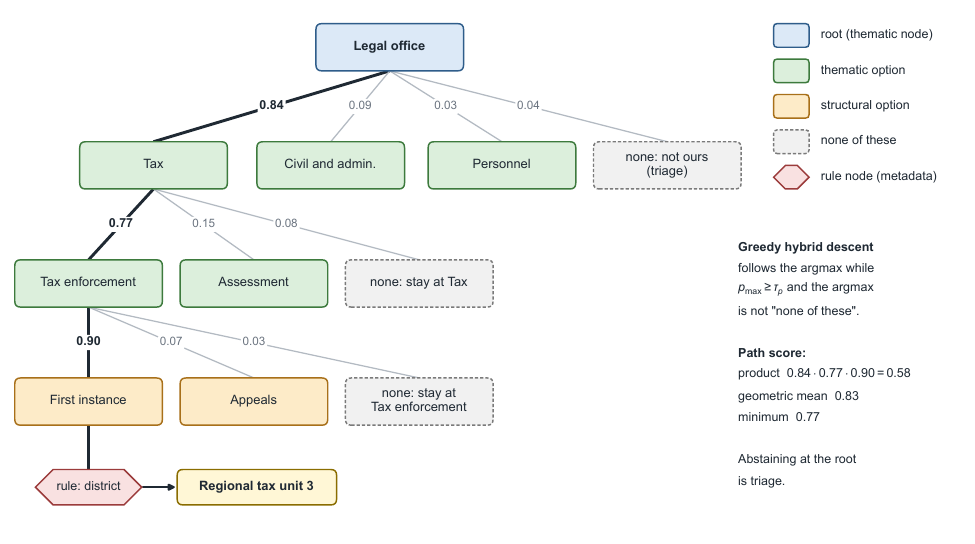}
  \caption{Routing as a descent over a fictitious organization chart. Every model node is a Choice over its children plus ``none of these''; structural options encode procedural attributes, and rule nodes apply a deterministic rule over metadata. The bold path is the greedy hybrid descent.}
  \label{fig:tree}
\end{figure}

Routing an item (here, a document) through an organization is a sequence of choices (Figure~\ref{fig:tree}). An organization chart is a tree whose internal nodes
are decisions over their children. We distinguish three kinds of nodes: \emph{thematic} nodes, whose children are
competences described in natural language; \emph{structural} nodes, whose children are defined by procedural attributes
(phase of the case, urgency regime, court instance) and recur across the tree with the same options; and \emph{rule} nodes,
whose choice is a deterministic rule over metadata (for example, the court district). Every model node is a Choice over its
children plus ``none of these''.

Abstention gives the tree its semantics. Stopping at a node means ``this belongs here but to none of the more specific
units''; abstaining at the root means ``this does not concern the institution'', which is \emph{triage}. Triage examples
must be faithful to that meaning: a communication that names the institution as a party to a lawsuit is, by definition,
the institution's business, whatever its subject. A synthetic triage example that changes only the subject while keeping
the institution as a party teaches the model a contradiction.

Greedy and beam descents over Choice probabilities, with beams ranked by the geometric mean of edge probabilities, are
documented for Jev \citep{typesafe2026hier}. What a narrow model adds is the meaning of abstention at every node, the
distinction between thematic, structural and rule nodes, and a descent fixed in advance for evaluation. Three search
procedures are defined over the tree, all using the same model:
\begin{itemize}[leftmargin=*,itemsep=1pt]
  \item \emph{greedy hybrid descent}, which follows the most probable child and stops when abstention wins or the top
        probability falls below a threshold; this is the procedure fixed in advance for evaluation;
  \item \emph{beam search} of small width;
  \item \emph{exhaustive sweep}, which scores every root-to-leaf path.
\end{itemize}
A path is scored by an aggregate of the probabilities along it: the product (the joint probability under conditional
independence), the geometric mean (insensitive to depth) or the minimum (the weakest link). The aggregate also serves as
the confidence of the final unit, for selective routing.

\section{Composed decisions}
\label{sec:composed}

\begin{figure}[t]
  \centering
  \includegraphics[width=\linewidth]{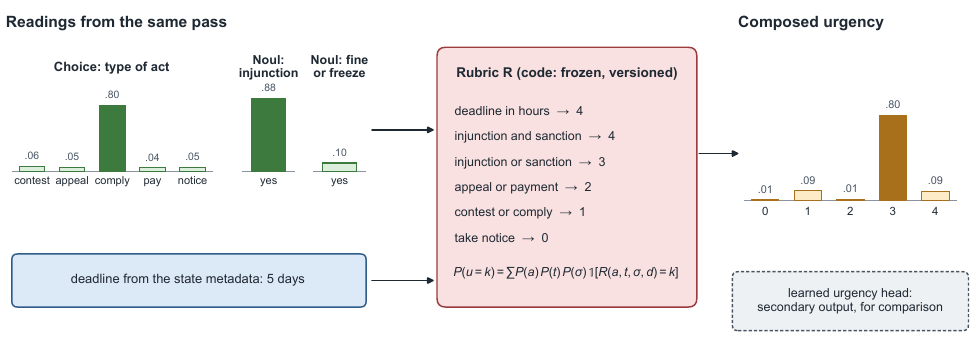}
  \caption{A composed decision. The urgency distribution is the push-forward, through a frozen rubric, of the model's own readings (type of act, injunction, sanction) and of the deadline in the state. The rubric shown is illustrative; the composed distribution is computed exactly from the readings displayed.}
  \label{fig:composed}
\end{figure}

Some answers are, by definition, a function of other answers. In our setting the urgency of a communication is a written
rubric over the type of act, the presence of an injunction against the institution, the presence and kind of a sanction,
and the deadline: an injunction gives at least level~3, an appeal or a payment without injunction or sanction gives
level~2, a deadline in hours gives level~4, and so on. A model can be trained to output urgency directly, as a sixth
judgment. It may then read every ingredient correctly and still combine them wrongly.

When a rubric $R$ defines the answer, we \emph{compose} instead (Figure~\ref{fig:composed}): the distribution of the composed answer is the push-forward
of the model's own readings through the rubric,
\begin{equation}
  P(u = k) = \sum_{a, t, \sigma} P(a)\, P(t)\, P(\sigma)\; \mathbf{1}\!\left[R(a, t, \sigma, d) = k\right],
\end{equation}
where $u$ is the urgency, $a$ ranges over the types of act, $t$ and $\sigma$ over the presence of an injunction and
of a sanction, and $d$ is the deadline, taken from the state. The readings come from questions the model answers anyway, so composition costs
nothing at inference. It makes the answer auditable (the rubric is code, versioned and frozen) and consistent with the
other answers by construction. The learned version is kept as a secondary output, to measure how much the model would
have learned on its own. Composition is not a substitute for reading: it moves the burden to the ingredients, which must
each be calibrated.

Decomposing a broad judgment into atomic questions and combining the answers in code is the recommended way of building
with System One models \citep{typesafe2026jev}. Composition as defined here is the probabilistic version of that
practice: the rubric acts on distributions, not on point answers, so the composed answer carries the uncertainty of its
ingredients and can itself be calibrated and thresholded. Urgency can also be framed as pairwise ranking of messages
\citep{gatto2026triage}; ranking orders a queue, while a rubric gives each item an absolute level, which is what
deadlines and routing rules consume.

\section{Training design}
\label{sec:training}

\subsection{Why not reinforcement learning}

Reinforcement learning with verifiable rewards is now a common way to adapt language models to tasks with checkable
answers \citep{deepseek2025r1}, and Jev is post-trained with an objective TypeSafe calls reinforcement learning for
calibrated decisions \citep{typesafe2026primer}, whose details are not public. An open reimplementation
\citep{maio2026evercld} samples an option $a \sim p$ and uses the reward $r_a = c_a - p_a$, where $c_a \in \{0, 1\}$ says
whether $a$ was right, in the line of proper-scoring rewards for verbalized confidence \citep{baniharouni2025rewarding};
its author notes that REINFORCE with this reward estimates half the gradient of the Brier score without bias. The
identity is short. Let $p$ be the predicted distribution over options and $c$ the one-hot target; the expected REINFORCE
gradient \citep{williams1992simple} is
\begin{equation}
  \mathbb{E}_{a \sim p}\!\left[r_a \nabla \log p_a\right]
  = \sum_a p_a (c_a - p_a) \nabla \log p_a
  = \sum_a (c_a - p_a) \nabla p_a
  = -\tfrac{1}{2}\, \nabla \sum_a (p_a - c_a)^2 .
\end{equation}
The expected policy gradient is exactly half the negative gradient of the Brier score. Reinforcement learning is the
right tool when only the outcome of the action taken is observed: a communication routed to one unit reveals whether that
unit was right, not which one would have been \citep{maio2026evercld}. The open encoder Laya uses a related scheme: it
adds Gaussian noise to the logits and rewards the perturbed distribution with a strictly proper score, through REINFORCE
\citep{convai2026laya}; in expectation, this follows the gradient of the score smoothed by the noise. When the label is
known, as in our synthetic data, the direct scoring rule reaches the same optimum without sampling variance, and it is
what we use. With a reward of $1$ for the correct answer and $0$ otherwise, as
in verifiable-reward training, the objective is accuracy rather than a proper scoring rule, and the resulting model is
expected to be overconfident.

\subsection{Curriculum}

\begin{figure}[t]
  \centering
  \includegraphics[width=\linewidth]{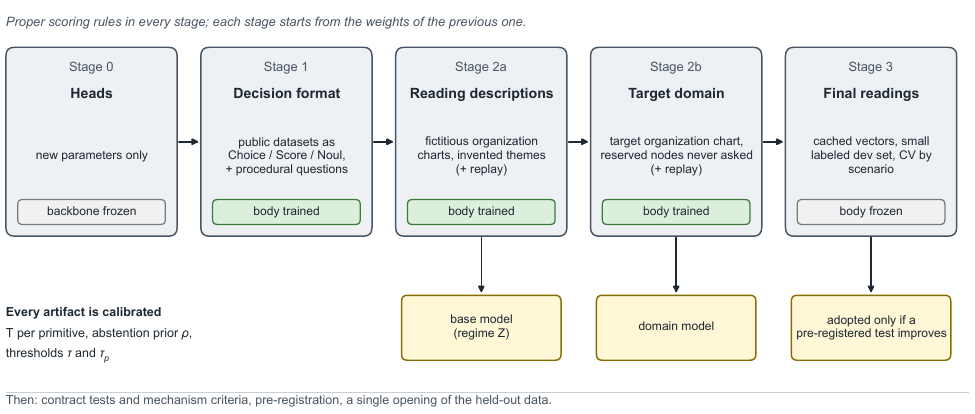}
  \caption{Training curriculum. Each stage starts from the weights of the previous one. Stage~2a never sees the target organization, so its artifact measures the skill of reading descriptions on options never seen (regime~Z).}
  \label{fig:curriculum}
\end{figure}

Training proceeds in stages, each starting from the weights of the previous one (Figure~\ref{fig:curriculum}).
\begin{description}[style=unboxed,leftmargin=1.2em,itemsep=2pt]
  \item[Stage 0: heads.] The backbone is frozen and only the new parameters are trained, so that the heads start from a
        reasonable point before the body moves.
  \item[Stage 1: decision format.] Public classification, intent and inference datasets are converted to the contract,
        with labels and their descriptions as options, and complemented by procedural generators of Choice, Score and Noul
        questions. This stage teaches the format, not the domain.
  \item[Stage 2a: reading descriptions.] Synthetic communications are routed through \emph{fictitious} organization charts,
        each drawn from a pool of invented themes with their own descriptions and lexicon. Because no option of the
        target organization appears, this stage can only be solved by reading descriptions; its output is a model that
        can be evaluated on the target organization without having seen it (zero-shot, ``regime Z'').
  \item[Stage 2b: domain.] Synthetic communications are routed through the target organization chart, with a reserved set
        of nodes never asked in training, so that the model's behavior on unseen options of the target domain can still be
        measured. A small fraction of public data is replayed in stages 2a and 2b.
  \item[Stage 3: final readings.] With the body frozen, the final linear readings are re-fitted on cached vectors using
        the small labeled development set, with cross-validation by scenario; the result is adopted only if a
        pre-registered test shows an improvement on a separate development set.
\end{description}
Each stage has a fixed budget of updates, periodic validation on a fixed sample, early stopping and restoration of the
best weights, and every artifact it produces is calibrated before it is used.

\subsection{The synthetic data engine}

\begin{figure}[t]
  \centering
  \includegraphics[width=\linewidth]{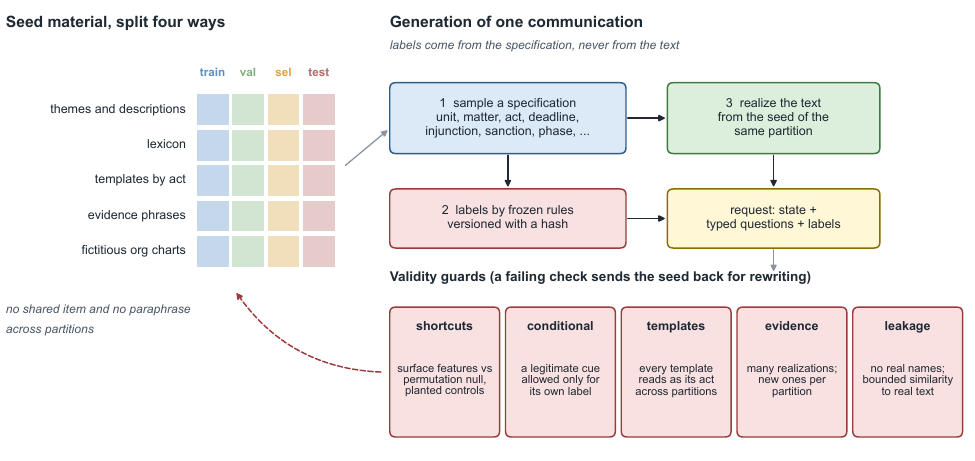}
  \caption{Synthetic data engine. The seed material is split four ways. Each communication starts from a sampled specification; its labels follow from frozen rules and its text is realized from the seed of the same partition. Validity guards send failing seed material back for rewriting.}
  \label{fig:engine}
\end{figure}

Labeled communications are scarce and partly protected, as is common in Portuguese legal text classification
\citep{noguti2024smallclaims}, so most training data is synthetic. The design rule is that
labels are produced by code, from a structured specification, and the text is realized from it, never the other way
around (Figure~\ref{fig:engine}).

\paragraph{Specification first.} For each communication the engine samples a specification: destination unit (a leaf of
the tree), legal matter, type of act, deadline, injunction, sanction, phase, regime, court instance, and metadata. Labels
for every question follow from the specification by frozen rules versioned with a hash; the urgency rubric of
Section~\ref{sec:composed} is one of them. In the full design, stage 2a draws about 2,700 communications over fictitious
charts of 12 to 30 nodes and stage 2b about 2,600 over the target chart; 30\% of the option names are replaced by opaque
codes, so that only the descriptions identify the options (Appendix~\ref{app:hiper}).

\paragraph{Seed material partitioned four ways.} Themes, lexicon items, text templates (openings, bodies by act,
closings), evidence phrases and fictitious organizations are written as seed material and split into training, validation,
selection and test partitions with no shared item and no paraphrase across partitions. A model that memorizes training
phrasings gains nothing on the other partitions.

\paragraph{Validity guards.} Synthetic data fails in characteristic ways, and each failure has a guard:
\begin{itemize}[leftmargin=*,itemsep=1pt]
  \item \emph{shortcuts}: a classifier trained on surface features (metadata, presence of a phrase, length) must not
        predict a label above a permutation null; planted controls verify that the test has power;
  \item \emph{conditional legitimacy}: a feature that legitimately carries a label (for example, the institution not
        being a party, for triage) is allowed for that label only, and a conditional test checks that, among examples
        sharing the feature, metadata alone cannot separate the label while the text can;
  \item \emph{unambiguous templates}: a classifier trained on the training partition must recognize the type of act of
        every template of the other partitions; a template that reads as a different act is rewritten;
  \item \emph{evidence inventories}: phrases that realize label-bearing evidence (an injunction granted, a fine) have
        many realizations in training and new ones in each evaluation partition, with no word that marks the opposite
        value in any partition;
  \item \emph{leakage}: no name, acronym or key of a real unit appears in the seed, and similarity to the real
        descriptions is bounded.
\end{itemize}

\section{Evaluation protocol}
\label{sec:protocol}

The protocol separates the acceptance of the \emph{mechanism} from the measurement of \emph{utility}, and both from the
single opening of held-out data.

\paragraph{Mechanism.} Before any held-out data is used, each artifact must pass:
\begin{itemize}[leftmargin=*,itemsep=1pt]
  \item \emph{contract tests}: differential tests of Propositions~1--3 (permuting options, removing questions, slicing,
        evaluating paths separately, padding, determinism) in double precision, with a tolerance set by whether the batch
        shape changes, and each paired with a negative control (an implementation without the mask, with sequential
        positions, or with visible siblings) that must fail;
  \item \emph{learning}: the validation loss of each primitive below both the uniform distribution and the end of stage 0;
  \item \emph{abstention}: non-zero recall of ``none of these'' at the first node of the tree on triage examples, with the
        false-alarm rate reported.
\end{itemize}

\paragraph{Utility.} Utility is measured as paired differences against trivial baselines (lexical retrieval over the
descriptions, the first option, a fallback to the parent unit), with confidence intervals from a bootstrap over scenarios,
not over individual decisions, because decisions from the same scenario are correlated. The primary hierarchical metric is
the accuracy of the final unit under the pre-registered greedy descent, macro-averaged over nodes. Two regimes are
reported: options already seen in training with new text (``regime E''), and options never seen (``regime Z''), measured
with the model trained only on fictitious organizations.

\paragraph{Locks.} The held-out partitions are opened once, after the mechanism has been accepted, the metrics and
decision rules have been pre-registered and the synthetic material has been reviewed by a domain expert. Queries to the
blind development set follow a closed list with a budget per experiment. Any artifact trained after the opening is
labeled as such and does not enter any acceptance criterion.

\paragraph{Architecture comparison.} \NDM-Block is the main line by design choice, not as a result. It is confronted with
\NDM-Pair under an equal budget of padded tokens, three seeds per arm and the same sampled options, and is kept if it is
not worse by more than a pre-registered margin, with the interval computed by a bootstrap over seeds and scenarios.

\section{Minimal validation}
\label{sec:validation}

This section checks the mechanism and the feasibility of the design on a small instance. It is not the evaluation of
Section~\ref{sec:protocol}: except for Section~\ref{sec:real}, it uses only public data and fictitious synthetic data,
and it never touches the target organization or any real communication. The reference instance, \texttt{ndm-base}, is mmBERT-base trained through stages
0, 1 and 2a of Section~\ref{sec:training} (public classification datasets, procedural questions and fictitious organization
charts) and calibrated on a synthetic validation partition. It has never seen the target organization. All measurements
ran on one GPU (NVIDIA RTX PRO 6000) in 49 minutes, with the procedure fixed in code before the run. The seed
material behind the synthetic data was drafted with the help of a language model and has not yet been reviewed by a
domain expert. Results on the target domain and on real data are the subject of the companion paper.

\subsection{Mechanism}

\paragraph{Contract tests.} Table~\ref{tab:contrato} reports the differential tests of Propositions~1--3 on
\texttt{ndm-base}, over synthetic requests of the selection partition, in float32 (all requests) and in float64 (a subset
of two requests), with the two attention implementations of the backbone. Every test stays below its tolerance: the
largest difference is $1.5 \times 10^{-5}$ in float32 and $3.2 \times 10^{-7}$ in float64, and several are exactly zero.
Every negative control fails by a wide margin: removing the mask, using sequential positions, letting siblings see each
other or giving a path full attention changes the outputs or the marker vectors by 2 to 14.

\begin{table}[t]
  \centering
  \small
  \begin{tabular}{@{}lccccl@{}}
    \toprule
                                        & \multicolumn{2}{c}{float32} & \multicolumn{2}{c}{float64} & \\
    \cmidrule(lr){2-3}\cmidrule(lr){4-5}
    test                                & sdpa & eager & sdpa & eager & negative control \\
    \midrule
    P1: permute the options             & 2.6e-6 & 4.7e-6 & 2.3e-9  & 1.6e-7 & sequential positions: 3.7 \\
    P2: question alone vs.\ batched     & 3.8e-6 & 3.5e-6 & 6.2e-9  & 3.2e-7 & no mask: 14.3 \\
    P2: state with and without questions & 0     & 0      & 0       & 0      & no mask: 2.0 \\
    P2: change a sibling question       & 0      & 0      & 0       & 0      & no mask: 3.9 \\
    P2: order of the questions          & 2.4e-6 & 2.7e-6 & 1.9e-8  & 1.3e-7 & sequential positions: 6.2 \\
    P3: slice the options               & 1.3e-5 & 1.3e-5 & 7.4e-14 & 1.2e-7 & visible siblings: 6.4 \\
    P3: tree as separate paths          & 1.5e-5 & 1.1e-5 & 9.3e-10 & 3.2e-7 & full-attention path: 8.0 \\
    padding                             & 2.7e-6 & 4.4e-6 & 4.4e-16 & 4.0e-7 & --- \\
    determinism                         & 0      & 0      & 0       & 0      & --- \\
    \bottomrule
  \end{tabular}
  \caption{Largest difference between two forms of the same request on \texttt{ndm-base}, and the smallest difference of
  the negative control. Tolerances: $10^{-4}$ in float32; in float64, $10^{-12}$ when the batch shape is kept and
  $10^{-6}$ (sdpa) or $10^{-5}$ (eager) when it changes. A control must differ by more than $10^{-3}$.}
  \label{tab:contrato}
\end{table}

\paragraph{State cache.} The cached path of Section~\ref{sec:arch} reproduces the full forward pass: the largest
difference in output probabilities is $7.3 \times 10^{-6}$ in float32 (168 questions over 20 states) and
$2.1 \times 10^{-8}$ in float64, while the two negative controls (local layers given the global mask; an option encoded
without the cached instruction) differ by more than 5. Encoding a state costs 6.5\,ms at 256 tokens and 59\,ms at 4,096;
an extra question batched against the cached state costs 0.8 to 4.9\,ms over the same range, about a tenth of the state.

\paragraph{Option order in practice.} On the 905 Choice questions of the selection partition, shuffling the options
leaves \texttt{ndm-base} unchanged: the argmax agrees in every question and the largest change of a probability is
$1.9 \times 10^{-6}$. An untrained decoder that scores each option by its log-likelihood (Qwen3-0.6B-Base) has about the
same accuracy in both orders (0.287 and 0.281), but changes its answer in 41\% of the questions, and single
probabilities move by up to 0.48: its errors are reshuffled, not removed.

\subsection{Reading descriptions}

\paragraph{Public data, by number of options.} Table~\ref{tab:k} shows accuracy on the official test splits of two
public intent datasets in Portuguese, with $K$ options drawn from the full label set (the correct one always included)
and read in one pass. Accuracy decreases slowly with $K$; with all 60 or 77 labels at once it is 0.74 and 0.70.

\begin{table}[t]
  \centering
  \small
  \begin{tabular}{@{}lrrrrrr@{}}
    \toprule
    dataset (labels) & $K = 4$ & $8$ & $16$ & $32$ & $60$ & $77$ \\
    \midrule
    MASSIVE intents & 0.956 & 0.924 & 0.878 & 0.811 & 0.740 & --- \\
    Banking77       & 0.953 & 0.913 & 0.860 & 0.801 & ---   & 0.701 \\
    \bottomrule
  \end{tabular}
  \caption{Accuracy of \texttt{ndm-base} on public test splits ($n = 2{,}974$ and $3{,}080$ sentences) by number of
  options read in one pass. 95\% bootstrap intervals are within $\pm 0.02$.}
  \label{tab:k}
\end{table}

\paragraph{Organization charts never seen.} On four fictitious organization charts of the selection partition, never
used in training (165 communications, triage excluded), the pre-registered greedy descent reaches the correct final unit
in 0.376 of the cases [0.297, 0.448], against 0.152 for lexical retrieval over the descriptions (BM25) and 0.067 for the
first option; the paired differences are $+0.224$ [0.133, 0.309] and $+0.309$ [0.224, 0.394], with intervals from a
bootstrap over communications. Beam search and the exhaustive sweep do not improve on the greedy descent (0.352). The
absolute level is modest, as expected of a model that has learned to read descriptions only on synthetic charts; the
point is that it decides on charts it has never seen, without retraining, well above the trivial baselines.

\subsection{Calibration}

The temperatures were fitted on a validation partition and are evaluated on the selection partition. For Choice, the
expected calibration error falls from 0.188 [0.160, 0.223] to 0.033 [0.033, 0.078], with average confidence moving from
0.865 to 0.698 for an accuracy of 0.678; for Score, from 0.088 to 0.064. For Noul it rises from 0.012 to 0.030: the
synthetic Noul questions are nearly separable, and the temperature fitted on one partition does not improve another.
On the public datasets, the synthetic temperature raises the negative log-likelihood (for example, from 0.123 to 0.232
with four MASSIVE intents): a temperature fitted on one domain does not transfer to another, which is why the domain
temperature of Section~\ref{sec:calibration} is fitted separately and reported next to the synthetic one.

\subsection{Heads and abstention}

\paragraph{Comparing candidates.} The same final reading, refitted out of fold on cached vectors, reaches a negative
log-likelihood of 1.181 (accuracy 0.655) when each option is scored independently, 1.266 (0.655) when the score also
receives summary statistics of the siblings (mean, maximum, $\log K$), and 0.799 (0.705) on the output of the set head.
Comparing candidates helps, and summary statistics are not a substitute for attention among them.

\paragraph{Reading abstention.} Three readings of ``none of these'' behave differently on the 627 Choice questions that
offer it (146 positives), with a threshold of 0.5. The probability of the explicit option is precise but conservative
(recall 0.25, false alarms 0.025, AUROC 0.79); the absolute adequacy is liberal (recall 0.66, false alarms 0.29, AUROC
0.76); the fallback to the parent is the weakest (AUROC 0.68). The two regimes, declared and undeclared, give similar
numbers, so the model does not depend on the declaration. Which reading to use is a trade-off between recall and false
alarms, set by the thresholds of Section~\ref{sec:calibration}.

\subsection{Block against pair}

The last experiment compares the two architectures of Table~\ref{tab:cost} under the rule of
Section~\ref{sec:protocol}, trained on public data and fictitious charts only: the same budget of 8 million padded tokens
per run, three seeds per arm and the same sampled options (the correct one and seven distractors). NDM-Block reaches a
lower negative log-likelihood than NDM-Pair on the selection partition, 0.819 against 1.024: the difference is
$-0.204$ [$-0.312$, $-0.107$] by a bootstrap over seeds and communications, so the upper bound is far below the margin of
0.02, and it is negative in every seed and for every question type (Choice $-0.19$, Score $-0.34$, Noul $-0.19$). The
comparison holds compute fixed, not data: with the same padded tokens, the block arm sees about 27,600 questions per run
and the pair arm about 4,300, because the pair re-encodes the state for every option. The result is the one that matters
for deployment, quality per unit of compute, rather than quality per example. In wall-clock time the block arms took
213\,s against 164\,s, because they carry more real tokens (5.7 against 4.9 million). The pre-registered comparison on
the target domain remains part of the companion study.

\subsection{From synthetic to real documents}
\label{sec:real}

The synthetic results were promising enough to train a first instance on real data. We summarize it only as a sign that
the design carries over; its evaluation belongs to the companion paper. The task is the type of a legal document: 28
kinds of procedural filings (initial petitions, answers, judgments, appeals, criminal pleadings and others), in about
840 thousand documents from court case files processed by our platform. The labels are those of the LLM-based classifier
in production, so the measure is agreement with that classifier, not accuracy against human annotation. The documents it
left as unknown (87\%), mostly certificates, notices and other filings outside the taxonomy, were kept out of training.
After deduplication and a split by case, so that no case falls in two partitions, the instance was continued from the
weights of the synthetic stage for one pass over 32 thousand class-balanced examples (50 minutes on one NVIDIA A100), as a
Choice question whose options are the 28 types with the descriptions of the production prompt, and calibrated on the
validation partition. On a held-out test set of 9,092 documents with the natural class distribution, it agrees with the
production classifier on 91.2\% of them (macro-F1 0.80, NLL 0.29, ECE 0.011), against 90.2\% (macro-F1 0.61, ECE 0.053)
for a TF-IDF and logistic-regression baseline trained on the same split; the difference in agreement is $+1.0$ point, with
a 95\% interval of $[+0.2, +1.9]$ by bootstrap over cases. The gain is in the rare types: on the eleven types with 1 to 50
test documents, the mean F1 is 0.75 against 0.37, while on the three largest both reach 0.92. Without ever seeing an
unknown document, the adequacy head places 58\% of a sample of 60 thousand of them below its threshold, that is, outside
the 28 types. These are the results of one run against labels produced by another model; they motivate the companion
study rather than replace it.

\section{Related work}

\paragraph{System One models.} Jev \citep{typesafe2026jev} introduced System One models and the interface adopted here:
a state, typed Choice, Score and Noul questions evaluated in parallel and in isolation, calibrated probabilities and a
confidence, no generated text. It is general-purpose and hosted, its architecture and training method are not public,
and its documentation describes patterns that we formalize for a narrow model: many questions per request, confidence-gated
routing, composite scoring, and hierarchical classification by greedy and beam search \citep{typesafe2026hier}. eve-rlcd
\citep{maio2026evercld} is an open, small decision model with the same interface: a decoder that prefills the state once
and scores declared options, trained from bandit feedback with an RLCD-style reward. Other decision models followed
within weeks, hosted and open, most of them adapted decoders that score option markers or letters
\citep{rafe2026benchmarking}. LLM-as-Jev \citep{li2026llmasjev} reads Choice, Score and Noul decisions from the next-token
probabilities of a causal LLM over bracketed option identifiers, whose tokens form a trie, and fine-tunes with a listwise
loss factorized over that trie; it treats position bias as a trade-off of joint scoring and mitigates it by shuffling the
options during training. Its tree is over the tokens of option identifiers; ours is over the request, and it is what makes
order invariance and question isolation hold by construction rather than on average. Open encoders are the closest to \NDM-Block. Laya \citep{convai2026laya} encodes each
question as its own sequence (instruction, options behind marker tokens, then the state, truncated to a fixed budget)
with ModernBERT or mmBERT, scores each option at its marker through a two-layer head, is trained with an RLCD-style
reward built from strictly proper scores, and is calibrated per question type and number of options; its released
checkpoints read the options sequentially and show a position effect, and an optional layout gives the options of a
question shared positions and forbids attention between them, which is our Proposition~1 for one question.
GLiNER2.5-Decide \citep{fastino2026glinerdecide} extends the schema-driven encoder GLiNER2 \citep{zaratiana2025gliner2} to
decisions and enforces constraints across them; its documentation does not describe invariance, calibration or
abstention. The open Decision family \citep{liu2026decision} has encoders derived from mmBERT that use Laya's layout with
one backbone path per question type, and causal decoders that read each option at its last token against a final query
token, calibrated by a post-hoc temperature; its report states that neither architecture enforces candidate-order
invariance, that a closed candidate set cannot express an abstention the application has not supplied, and that state
computation is not reused across questions. Decision-2.0-Kai \citep{vllmsr2026decision2kai}, the 0.6B model of the next
release, is a causal decoder that can share the state prefix across the questions of a request at inference, which in a
causal model is prefix caching, and still reads the options of a question in sequence, each seeing those before it.
\NDM-Block differs from all of these in serializing every question and every option of a request into one tree over a
state encoded once and cached, so that order invariance, question isolation and exact slicing hold by construction,
proven and tested on a trained model against negative controls; in its explicit abstention, read through a set head and
a sibling-blind adequacy score; in direct proper scoring, where Laya and eve-rlcd follow a policy gradient; and in being
narrow, with a training pipeline for one organization's decisions.
Independent evaluations support the narrow route: a DeBERTa-v3-large classifier fine-tuned on each task is the most
accurate model on almost every labeled benchmark of the first comparison of System One models, where attaching yes and
no to the opposite rubrics changes 50.5 in 100 answers of the hosted model, against at most 6.5 for fine-tuned decision
checkpoints \citep{rafe2026benchmarking}.

\paragraph{Pairwise and dual encoders.} Cross-encoders score a query and a candidate jointly \citep{nogueira2019passage};
dual encoders embed them separately and compare \citep{reimers2019sbert}. Zero-shot classification by entailment treats
each label description as a hypothesis \citep{yin2019benchmarking}. All three score candidates one at a time and produce
scores, not distributions over a set; \NDM-Block reads all candidates of all questions in one pass, lets candidates of the
same question be compared, and is trained for calibrated distributions.

\paragraph{Language models as selectors.} Prompted language models show selection bias towards option positions and
labels \citep{zheng2024mcq} and sensitivity to where evidence sits in long contexts \citep{liu2024lost}. These are
properties of the interface as much as of the model; the shared logical positions of \NDM-Block remove the first by
construction.

\paragraph{Masks, shared positions and segment caches.} Tree-structured attention masks verify several candidate
continuations in one pass in speculative decoding \citep{miao2023specinfer,cai2024medusa}, and shared-prefix attention
improves serving throughput \citep{juravsky2024hydragen}. Prompt Cache \citep{gim2024promptcache} and Block-Attention
\citep{ma2025blockattention} encode segments independently to reuse their keys and values. Parallel Context Windows
\citep{ratner2023pcw} reuse positions across windows of a long context, and PINE \citep{wang2025pine} makes attention
between segments bidirectional and position-free to remove position bias without training. For candidates, UniMC
\citep{yang2022unimc} gives each option of a multiple-choice encoder its own position ids and an attention mask that
keeps options apart, Set-LLM \citep{egressy2025setllm} guarantees permutation invariance of a decoder over a set through
a set mask and set positions, and Laya offers the same for the options of one question \citep{convai2026laya}. The
option-level mechanism is therefore not new. What \NDM-Block adds is its use at every level of a request, as the request's
semantics, in training and inference: questions isolated from each other as options are, a state that attends only to
itself and can be cached exactly, and exact slicing of options and paths, each a tested property of a trained model.

\paragraph{Sets.} Permutation-invariant and equivariant functions over sets \citep{zaheer2017deepsets} inform the set
head, which compares the options of one question without positional information. Set-Encoder
\citep{schlatt2025setencoder} lets candidate passages interact inside a re-ranking cross-encoder without depending on
their order; the set head does the same after the encoder, over one vector per option.

\paragraph{Calibration and selective prediction.} Temperature scaling \citep{platt1999probabilistic,guo2017calibration},
proper scoring rules \citep{brier1950verification,epstein1969scoring,gneiting2007strictly} and selective classification
\citep{geifman2017selective} are standard; we combine them per question type and add the safeguards that a decision
component needs (bounded temperatures, abstention priors, domain temperature as a secondary report). Strictly proper
scoring rules have also been used as training objectives for language generation \citep{shao2024proper}; here they are
the whole objective, because the output is a distribution over a closed set.

\paragraph{Encoders.} The backbone builds on bidirectional encoders \citep{devlin2019bert} and their modern long-context,
multilingual successors \citep{warner2024modernbert,marone2025mmbert}, with rotary positions \citep{su2021roformer}, which
make logical positions inexpensive.

\paragraph{System 1 and System 2.} The distinction between fast, automatic and slow, deliberate processing
\citep{kahneman2011thinking} has motivated proposals for deliberate, attention-based reasoning in deep learning
\citep{bengio2017consciousness}, and gave System One models their name \citep{typesafe2026jev}. \NDM{}s sit at the fast
end: they are meant to be the fast component, with the
calibration and abstention that let a slower component, or a person, be called only when needed.

\section{Limitations and open questions}
\label{sec:limitations}

\begin{itemize}[leftmargin=*,itemsep=2pt]
  \item \textbf{Synthetic-to-real gap.} Most training data is synthetic. The guards of Section~\ref{sec:training} remove
        known shortcuts but cannot certify realism; the regime-E and regime-Z measurements on real held-out data are the
        only answer, and they are deliberately few.
  \item \textbf{Rubrics.} Composed decisions are only as good as the rubric and its ingredients. A rubric that does not
        match the institution's practice will be applied faithfully and wrongly.
  \item \textbf{Depth.} Accuracy per decision compounds along a path. A model that is right on most individual decisions
        can still be wrong on the final unit; final-unit accuracy, not per-decision accuracy, is the honest metric.
  \item \textbf{Calibration under shift.} Temperatures fitted on synthetic data may not transfer; the domain temperature
        is reported for this reason, and separable calibration sets are flagged rather than trusted.
  \item \textbf{Content of the state.} The structural invariances cover the order of the options and the presence of
        other questions, not the content of the state: short, natural additions to the context can flip the answers
        of System One models \citep{xu2026jevout}, and nothing in the architecture prevents it. Robustness to such
        additions is untested.
  \item \textbf{Floating-point equivalence.} The invariances are exact in exact arithmetic; in practice they hold up to a
        tolerance that depends on batch shape and attention implementation, and the tolerance itself must be justified.
  \item \textbf{Wording of abstention.} The abstention option has a fixed description in each data source. The
        undeclared regime tests recognition by description, but not robustness to paraphrases of ``none of these'';
        paraphrase augmentation of that option is future work.
  \item \textbf{Threshold guarantees.} The selective thresholds are empirical. Risk-controlling alternatives with
        finite-sample guarantees need exchangeable calibration data, which only real data can supply.
  \item \textbf{Narrow by design.} An \NDM{} serves the decisions it was trained for. A new kind of decision needs its
        own seed material and training; until then, or when the decisions are too varied to justify a pipeline, a
        general-purpose System One model is the better tool. Where the line falls is an empirical question.
  \item \textbf{Kinds of state.} Nothing in the architecture is specific to documents, but it has only been exercised on
        text and on records that carry text: public intent sentences, synthetic court communications with their
        metadata and, in Section~\ref{sec:real}, real legal documents. States of other kinds (transactions, event logs, the context of an agent) need their own training data
        and evaluation.
  \item \textbf{Scope.} The design targets narrow, repetitive decisions with written taxonomies. It is not a reasoning
        system, and the decisions it makes in a legal context are recommendations to people, not substitutes for them.
\end{itemize}

\section{Conclusion}

A decision component should answer typed questions about a state with calibrated probabilities, in one pass, without being
swayed by the order of its options or by the other questions it is asked, and it should be able to say that none of the
options applies. We have specified an architecture in which the most fragile of these properties are structural, a way of
deciding over hierarchies and of composing decisions defined by rubrics, a training design that reaches the domain only
after the decision format and the skill of reading descriptions, and a protocol that accepts or rejects the mechanism
before the held-out data is opened. The minimal validation shows that the mechanism holds in a trained model and that the
design is feasible and efficient at small scale, and a first instance trained on real legal documents is promising;
whether it delivers on real decisions in production is the subject of the companion paper.

\section*{Use of generative AI}

The authors used Claude (Anthropic), a large language model, to help draft and edit the text, implement and test the
code, and produce the figures. The seed material for the synthetic data was also drafted with its help, as stated in
Section~\ref{sec:validation}. The authors reviewed the content and take full responsibility for it.

\bibliographystyle{plainnat}
\bibliography{refs}

\appendix

\section{Building the block mask and the logical positions}
\label{app:montagem}

The input is a list of tokenized requests: a state $S$ and, for each question $q$, an instruction $I_q$ and options
$O_{q1}, \dots, O_{qK_q}$ (levels for Score; none for Noul, whose criteria sit inside $I_q$). Length buckets are
$\mathcal{L} = \{256, 512, 768, 1024\}$ tokens.

\begin{enumerate}[leftmargin=*,itemsep=2pt]
  \item \textbf{Blocks.} For each request, open a sequence and emit the state block, with parent $-1$ and logical
        positions $0, \dots, |S|-1$.
  \item \textbf{Questions.} For each question, if the instruction and its first option do not fit in the space left
        under $\max \mathcal{L}$, open a new sequence and copy the state block into it. Emit the instruction block, with
        the state block as parent and positions $|S|, \dots, |S| + |I_q| - 1$.
  \item \textbf{Options.} For each option, if it does not fit, open a new sequence and copy the state and the instruction
        blocks, with the same positions (a slice). Emit the option block, with the instruction block as parent and
        positions $|S| + |I_q|, \dots, |S| + |I_q| + |O_{qk}| - 1$. Sibling options therefore share positions.
  \item \textbf{Visibility.} For each sequence, $\mathrm{ve}[i, j] = 1$ if block $j$ is block $i$ or one of its
        ancestors, by walking up the parent chain from each block.
  \item \textbf{Mask.} $M_{tu} = 0$ if $\mathrm{ve}[\mathrm{block}(t), \mathrm{block}(u)]$ and $-\infty$ otherwise;
        padding belongs to no block and sees only itself. Local layers also require $|\mathrm{pos}(t) - \mathrm{pos}(u)|
        \leq w$, on logical positions.
  \item \textbf{Batches.} Pad each sequence to its length bucket and group sequences of the same bucket into micro-batches
        of about 4,096 padded tokens. Each question is one row of the readout heads, with (sequence, position) pointers to
        its markers; the pointers of a sliced question cross sequences of the same batch.
\end{enumerate}

\paragraph{Cost.} Building $\mathrm{ve}$ takes $O(N_B \cdot d)$ per sequence, for $N_B$ blocks and tree depth $d \leq 3$
inside a request; the token mask takes $O(L^2)$ per sequence, the same order as dense attention itself. Each extra slice
re-encodes $|S| + |I_q|$ tokens. Logical positions are integers and enter the rotary embeddings directly, so they cost
nothing.

\clearpage
\section{Sizes, weights and volumes}
\label{app:hiper}

\begin{table}[h!]
  \centering
  \small
  \begin{tabular}{@{}p{0.27\linewidth}p{0.68\linewidth}@{}}
    \toprule
    Component & Value \\
    \midrule
    Backbone & mmBERT-base: 22 layers (8 global, 14 local, window $\pm 64$ logical positions), $d = 768$; frozen
               embeddings (196.6\,M), trained body (110.3\,M) \\
    Markers & 9 used and 7 reserved vocabulary entries; offsets $\Delta \in \mathbb{R}^{17 \times 768}$, initialized to the
              mean embedding of an equivalent phrase \\
    Question and option vectors & $W_q: \mathbb{R}^{1536} \to \mathbb{R}^{384}$ on $[h_S; h_Q]$ plus a type embedding;
              $W_o: \mathbb{R}^{768} \to \mathbb{R}^{384}$ \\
    Set head & 2 pre-norm layers, width 384, 6 heads, feed-forward 1,536; output projections initialized at zero \\
    Relative scorer & $\mathrm{LN}[o; q; o \odot q]$ (1,152) $\to$ linear (384) $\to$ GELU $\to$ one linear reading per type \\
    Adequacy & $\mathrm{LN}[h_O; h_Q; h_O \odot h_Q]$ (2,304) $\to$ linear (384) $\to$ GELU $\to$ linear \\
    Noul & $\mathrm{LN}[h_Q; h_S]$ (1,536) $\to$ linear (384) $\to$ GELU $\to$ linear \\
    New parameters & 6.38\,M \\
    Losses & Choice: log + Brier (weight 1; accepted set merged); Score: log + RPS (1); Noul: log + Brier (1); adequacy:
             binary log-loss (0.2) \\
    Optimizer & AdamW, $\beta = (0.9, 0.98)$, $\epsilon = 10^{-6}$, weight decay 0.01 (none on norms, biases, $\Delta$ and
                type embeddings), gradient clipping 1.0, cosine schedule with floor 0.1, layer-wise decay 0.95; learning
                rate $2 \times 10^{-5}$ for the body and $2 \times 10^{-4}$ for the new parts ($5 \times 10^{-4}$ in stage 0) \\
    Length limits & state 640 tokens in training and 1,728 in inference; instruction 192; option 128; branch 1,024 and
                    2,048; buckets of 256 to 1,024 tokens; about 4,096 padded tokens per micro-batch \\
    Curriculum & micro-batches: stage 0, 600; stage 1, 6,000 (about 48,000 public and procedural sequences); stage 2a,
                 2,400 (about 2,700 communications); stage 2b, 2,400 (about 2,600) \\
    Synthetic engine & seed split in 4 partitions; fictitious charts of 12 to 30 nodes; 10\% probes with known posterior;
                       30\% opaque option names; abstention declared for about half of the questions (label-independent
                       hash); similarity to real descriptions capped at 0.5 \\
    Calibration & $T \in [0.25, 10]$ per primitive; domain tolerance $\tau = \ln 1.5$; risk target 5\% for $\tau$ and
                  $\tau_p$ \\
    \bottomrule
  \end{tabular}
  \caption{Sizes, weights and volumes of the reference configuration.}
  \label{tab:hiper}
\end{table}

\end{document}